\pdfoutput=1
\documentclass[11pt]{article} 
\def\arXiv{1} 

\usepackage{enumitem}
\usepackage{amssymb}
\usepackage{amsbsy}
\usepackage{amsmath}
\usepackage{amsfonts}
\usepackage{latexsym}
\usepackage{graphicx}
\usepackage{color}
\usepackage{float}
\usepackage{xifthen}
\usepackage{xspace}
\usepackage{mathtools}
\usepackage{bbm}
\usepackage{multirow}
\usepackage[normalem]{ulem}
\usepackage{array}
\usepackage[utf8]{inputenc}
\usepackage[T1]{fontenc}
\usepackage{etoolbox}
\newtoggle{restatements}

\newtoggle{heavyplots}

\usepackage{xfrac}
\newcommand{\notarxiv}[1]{foo}
\newcommand{\arxiv}[1]{ba}
\ifdefined \arXiv
	\renewcommand{\arxiv}[1]{#1}%
	\renewcommand{\notarxiv}[1]{}%
\else%
	\renewcommand{\arxiv}[1]{}%
	\renewcommand{\notarxiv}[1]{#1}%
\fi%

\usepackage{titlesec}

\usepackage{amsthm}
\usepackage{thmtools}
\usepackage{thm-restate}
\usepackage{makecell}

\notarxiv{
\usepackage[utf8]{inputenc} %
\usepackage[T1]{fontenc}    %
\usepackage{url}            %
\usepackage{booktabs}       %
\usepackage{amsfonts}       %
\usepackage{nicefrac}       %
\usepackage{microtype}      %
\usepackage{xcolor}
\usepackage{stackengine}
\usepackage{comment}
}

\arxiv{
	\usepackage{booktabs}  
	\usepackage[dvipsnames]{xcolor}
	\usepackage[top=1in, right=1in, left=1in, bottom=1in]{geometry}
	\usepackage[numbers,square]{natbib}
	\usepackage{url}
	\usepackage[hidelinks]{hyperref}
	\hypersetup{
		colorlinks=true,
		linkcolor=blue!70!black,
		citecolor=blue!70!black,
		urlcolor=blue!70!black}
}

\newtheorem{lemma}{Lemma}

\newtheorem{corollary}{Corollary}

\newtheorem{remark}{Remark}
\newenvironment{proof-sketch}{\noindent{\bf Sketch of Proof}\hspace*{1em}}{\qed\bigskip}

\newtheorem{assumption}{Assumption}

\arxiv{
}

\makeatletter
\newcommand*{\propenum}[1]{%
	\expandafter\@propenum\csname c@#1\endcsname%
}

\newcommand*{\@propenum}[1]{%
	$\ifcase#1\or1\or1'\or2\or3\or42%
	\else\@ctrerr\fi$%
}
\AddEnumerateCounter{\propenum}{\@propenum}{1}
\makeatother

\DeclarePairedDelimiter{\norm}{\|}{\|}

\NewDocumentCommand\Ex{s O{} m }{%
	\mathbb{E}%
	\begingroup
	\IfBooleanTF{#1}
	{\ExInn*{#3}}
	{\ExInn[#2]{#3}}%
	\endgroup
}

\DeclarePairedDelimiterX\ExInn[1]{[}{]}{%
	\activatebar
	#1%
}

\RenewDocumentCommand\Pr{sO{}r()}{%
	\mathbb{P}%
	\begingroup
	\IfBooleanTF{#1}
	{\PrInn*{#3}}
	{\PrInn[#2]{#3}}%
	\endgroup
}

\DeclarePairedDelimiterX\PrInn[1](){%
	\activatebar
	#1%
}

\newcommand{\activatebar}{%
	\begingroup\lccode`~=`|
	\lowercase{\endgroup\def~}{\;\delimsize\vert\;}%
	\mathcode`|=\string"8000
}

\newcommand{\linf}[1]{\norm{#1}_\infty} %
\newcommand{\normbigg}[1]{\bigg\|{#1}\bigg\|} %
\newcommand{\lonebigg}[1]{\normbigg{#1}_1} %
\newcommand{\linfbigg}[1]{\normbigg{#1}_\infty} %
\newcommand{\norms}[1]{\|{#1}\|} %
\newcommand{\lones}[1]{\norms{#1}_1} %
\newcommand{\R}{\mathbb{R}} %
\newcommand{\E}{\mathbb{E}} %
\renewcommand{\P}{\mathbb{P}}	%

\makeatletter
\let\oldnl\nl%
\newcommand{\nonl}{\renewcommand{\nl}{\let\nl\oldnl}}%
\makeatother

\DeclareMathOperator*{\argmin}{arg\,min}

\newcommand{\defeq}{\coloneqq}

\newcommand{\xset}{\mathcal{X}}

\newcommand{\eps}{\epsilon}

\renewcommand{\O}[1]{O\left( #1 \right)}

\newcommand{\Otilb}[1]{\widetilde{O}\left( #1 \right)}

\NewDocumentCommand{\prox}{ m m 
O{\alpha}}{\mathrm{Prox}_{#1}^{#3}(#2)}

\newcommand{\veps}{\varepsilon}

\newcommand{\calM}[0]{\mathcal{M}}
\newcommand{\calA}[0]{\mathcal{A}}
\newcommand{\calS}[0]{\mathcal{S}}

\newcommand{\PP}[0]{\mathbf{P}}
\newcommand{\0}[0]{\mathbf{0}}
\newcommand{\M}[0]{\mathbf{M}}

\newcommand{\SSS}[0]{\mathbf{S}}

\newcommand{\xx}[0]{\mathbf{x}}

\newcommand{\vv}[0]{\mathbf{v}}

\newcommand{\rr}[0]{\mathbf{r}}
\newcommand{\qq}[0]{\mathbf{q}}
\newcommand{\nnu}[0]{{\boldsymbol{\nu}}}

\newcommand{\ee}[0]{\mathbf{e}}

\newcommand{\tmix}{t_{\mathrm{mix}}}

\newcommand{\1}{\mathbf{1}}

\newcommand{\calI}{\mathcal{I}}

\newcommand{\A}{\mathrm{A_{tot}}}

\newcommand{\pp}{\mathbf{p}}
\newcommand{\MM}{\mathbf{M}}
\newcommand{\NN}{\mathbf{N}}
\newcommand{\q}{\mathbf{q}}
 
\notarxiv{
	\usepackage[subtle, all=normal, mathspacing=tight, floats=tight, 
	sections=tight]{savetrees}
}

\title{Towards Tight Bounds on the Sample Complexity\\ of Average-reward MDPs}

\arxiv{
	
\author{
Yujia Jin \\
Stanford University \\
\texttt{\href{mailto:yujiajin@stanford.edu}{yujiajin@stanford.edu}}
\and
	Aaron Sidford \\
Stanford University \\
\texttt{\href{mailto:sidford@stanford.edu}{sidford@stanford.edu}}
}

\date{}
}

\begin{document}

\maketitle

\begin{abstract}
We prove new upper and lower bounds for sample complexity of finding an $\eps$-optimal policy of an infinite-horizon average-reward Markov decision process (MDP) given access to a generative model. When the mixing time of the probability transition matrix of all policies is at most $\tmix$, we provide an algorithm that solves the problem using $\widetilde{O}(\tmix \eps^{-3})$ (oblivious) samples per state-action pair. Further, we provide a lower bound showing that a linear dependence on $\tmix$ is necessary in the worst case for any algorithm which computes oblivious samples. We obtain our results by establishing connections between infinite-horizon average-reward MDPs and discounted MDPs of possible further utility. 
\end{abstract}

\newpage
\tableofcontents
\newpage

\section{Introduction}

In this paper we consider the fundamental problem of computing an approximately optimal policy in a Markov decision process (MDP) given by a generative model. We consider a standard MDP model with a known set of states and actions. If an agent chooses an action at a given state, a known reward is immediately given to the agent and the agent probabilistically transitions to a new state by an (unknown) fixed distribution (as a function of the state-action pair). Given access to a \emph{generative model}~\citep{K03}, i.e. an oracle which when queried by a state-action pair returns an independent sample from the distribution over next states, our goal is to find a policy, i.e. a choice of action per state, that approximately maximizes a measure of (cumulative) reward over time, e.g.  discounted reward, average reward, etc.

Solving MDPs with a generative model is a fundamental problem in learning theory and a classic model for decision making under uncertainty and reinforcement learning~\citep{puterman2014markov,sutton2018reinforcement}. It is a prominent theoretical test-bed for learning algorithms and has been studied extensively over past 20 years~\citep{K03}. Multiple algorithms have been proposed for the problem~\citep{puterman2014markov} and there has been extensive research on improving the sample complexity for finding an approximately optimal policy~\citep{W17d, SWWYY18, W19, agarwal2019optimality}.
 
In certain settings, the sample complexity of the problem is settled. For example, consider discounted  MDPs (DMDPs) where the goal is to minimize the $\gamma$-discounted reward, i.e. the sum of the rewards where the reward at step $t\ge 0$ is discounted by $\gamma^t\in[0,1)$. For a DMDP with $\A$ total state-action pairs, it is known that $\widetilde{\Omega}(\A/(1-\gamma)^3\eps^2)$~\footnote{Many results in literature instead consider MDPs where at each state there are a fixed $A$ actions and use $|\calS|A$ to denote the total number of state-action pair in their result. These results typically generalize to non-uniform action set among states, and thus we replace this total number of state-action pair with $\A$ throughout.} samples are necessary to find an $\eps$-optimal policy~\cite{azar2013minimax,feng2019does} and there are algorithms which solve the problem with near-optimal $\widetilde{O}(\A/(1-\gamma)^3\eps^2)$ samples~\cite{SWWYY18,agarwal2019optimality,li2020breaking} for certain ranges of $\eps$. Similarly, in the case of finite-horizon MDPs, again optimal sample and query complexities are known~\cite{SWWYY18} for sufficiently small $\epsilon$.
 
Another popular class of MDPs are infinite-horizon average-reward MDPs~(AMDPs)~\cite{mahadevan1996average}. Here the reward function (also known as the \emph{gain} for AMDPs) is the infinite-horizon average reward and it is assumed that for any given policy the transition matrix it induces has a mixing time bounded by $\tmix > 0$. AMDPs arise naturally in controlling computer systems and communication networks, where a controller makes frequent decisions, and for inventory systems with frequent restocking decisions~\cite{puterman2014markov}.  However, despite advancements in the theory of DMDPs and multiple proposed algorithms~\cite{W17m,jin2020efficiently}, the optimal sample complexity of AMDPs has resisted similar characterization.The best known sample complexity is $\widetilde{O}(\A\tmix^2/\eps^2)$ \cite{jin2020efficiently} and there is no known lower-bound. 

In comparison to existing near-optimal DMDP methods for  generative models~\citep{K03}, recent methods for provably solving AMDPs suffer from two additional limitations.%
~First, they work with randomized policies, i.e. ones which choose a distribution over actions in every state. Correspondingly, these methods require stronger condition of mixing bound on all randomized policies.%
 Second, they use dynamic samples from the model, as opposed to the case for DMDPs where the sampling process can be completely oblivious, i.e. a fixed number of samples can be generated per state-action pair.

In this paper we make progress on these problems. First we provide a new method that finds an $\eps$-optimal deterministic stationary policy using $\widetilde{O}(\A\tmix/\eps^3)$ oblivious samples, thereby overcoming these limitations of previous methods. Further, we provide a lower bound showing that $\Omega(\A\tmix/\eps^2)$ oblivious samples are necessary. Consequently, we resolves the question of optimal dependence on $\tmix$ for oblivious sampling methods. We achieve these results by establishing a connection between infinite-horizon average-reward MDPs and discounted MDPs which may be of utility for further research in this area.

\subsection{Problem Setup}\label{sec:intro-setup}

We define a Markov decision process (MDP) as a tuple $\calM\defeq (\calS,\calA,\PP,\rr)$ with the following interpretations: 
\begin{itemize}
\item 	state space $\calS$ - a finite discrete set of states with size $|\calS|$
that the process transits on. In particular, $s\in\calS$ denotes a single state. 
\item  total action space $\calA$ - the union of all actions that an agent can take at any state $s\in \calS$, i.e. $\calA = \cup_{s\in\calS} \calA_s$, where $\calA_s$ is the action space at state $s$ (and the $\calA_s$ are disjoint). We denote the total number of state-action pairs as $\A = \sum_{s\in\calS}|\calA_s|$.
\item transition probability matrix  $\PP\in\R^{\calA\times \calS}$ - When choosing action $a_s \in \calA_s$ at state $s \in \calS$ the next state is chosen from the distribution $\pp_{s,a_s} = \PP((s,a_s),\cdot)\in\Delta^{\calS}$.  We assume that $\PP$ is unknown but can be queried by a \emph{generative model} which when queried at any state action pair $(s,a_s)$ outputs an independent sample from the distribution $\pp_{s,a_s}$.
\item reward vector $\rr\in [0,1]^{\calA}$ - the reward at any state $s$ when playing action $a_s$ is denoted as $r_{s,a_s}$. As is common practice we assume the reward doesn't depend on the state it transits to~\citep{SWWYY18}, and is bounded in $[0,1]$ (by uniform rescaling).
\end{itemize}
Given MDP $\calM$ at any state $s\in\calS$ an agent can take action $a_s\in\calA_s$ after which it receives instant reward $r_{s,a_s}$ and transits to some other state $s'$ with probability $\pp_{s,a_s}(s')$. A \emph{(deterministic stationary) policy} of an MDP is a mapping $\pi:\calS\rightarrow \calA$, i.e. it maps  each state $s\in\calS$ to a fixed action $\pi(s) \in \calA_s$. A \emph{randomized stationary policy} of an MDP is a mapping $\pi:\calS\rightarrow  \Delta^{\calA}$, i.e. it maps a state $s \in \calS$ to a fixed distribution over actions $\pi(s) \in \Delta^{\calA_s}$. Under a given fixed policy $\pi$ and initial distribution over states $\mathbf{q} \in \Delta^\calS$,
 the MDP generates a sample path $\{(s_1,a_1)\in(\calS\times\calA),(s_2,a_2)\in(\calS\times\calA),\cdots\}$ where $a_i$ is chosen from $\pi(s_i)$ and $s_{i+1}$ is chosen from each $\pp_{s_i,a_i}$. Following its sample path, it receives the (cumulative) reward $V_\mathbf{q}^\pi$ defined as 
\begin{equation}\label{eq:reward}
\begin{aligned}
V_\mathbf{q,\gamma}^\pi = 
 & \E^\pi\left[\sum_{t\ge 1}\gamma^{t-1}r_{s_t,a_t}|s_1\sim \mathbf{q}\right], & \gamma\text{-DMDPs};\\
V_\mathbf{q}^\pi = & 	\lim_{t\rightarrow \infty}\frac{1}{T}\E^\pi\left[\sum_{t\ge 1}r_{s_t,a_t}|s_1\sim \mathbf{q}\right], & \text{AMDPs},
\end{aligned}
\end{equation}
given $\mathbf{q}$ as the initial distribution over states. We consider discount factor $\gamma<1$ for DMDPs, and an average reward for AMDPs. For brevity,  under a policy $\pi$ we use $\PP^\pi\in\R^{\calS\times\calS}$ to denote the transition matrix of the underlying Markov chain where $\PP^\pi(s,s') = \sum_{a_s\in\calA_s}\pi_s(a_s)\pp_{s,a_s}(s')$, and $\rr^\pi\in[0,1]^\calS$ is the reward vector where $\rr^\pi(s) = \sum_{a_s\in\calA_s}\pi_s(a_s)r_{s,a_s}$. We define $\nnu^\pi$ to denote the stationary distribution of a transition matrix $\PP^\pi$ under policy $\pi$, i.e. $(\nnu^\pi)^\top\PP^\pi = \nnu^\pi$.

Our goal in solving an MDP is to find an optimal policy $\pi^\star$, that maximizes the cumulative reward. We say (stationary) policy $\pi$ is $\eps$-optimal, if
$V^\pi_\q\ge V^{\pi^\star}_\q-\eps$, for any initial distribution $\q \in \Delta^{\calS}$.
In this paper our main goal is to characterize the number of samples that need to be collected for each state-action pair (in the worst case) to find an $\eps$-optimal policy for AMDPs. We further restrict our attention to mixing AMDPs, which we define as those satisfying the following assumption.

\begin{assumption}
\label{assum}
	An AMDP instance is \emph{mixing} if for any policy $\pi$, there exists a stationary distribution $\nnu^\pi$ so that for any initial distribution $\mathbf{q}\in\Delta^\calS$, the induced Markov chain has mixing time bounded by $\tmix<\infty$ , where $\tmix$ is defined as 
	\begin{equation*}\label{def-mixingtime}
	\tmix\defeq\max_{\pi}\left[\argmin_{t\ge1 }\left\{\max_{\qq\in\Delta^\calS} \lones{({\PP^{\pi}}^{\top})^t\qq-\nnu^\pi}\le \tfrac{1}{2}\right\}\right].
	\end{equation*}
\end{assumption}

This is a natural and widely used regularity assumption~\citep{W17m,jin2020efficiently} for AMDPs to ensure the existence of cumulative reward of AMDPs in~\eqref{eq:reward}. Under Assumption~\ref{assum} one can show that $V^\pi = \langle\nnu^\pi, \rr^\pi\rangle$, where $\nnu^\pi$ is the stationary distribution over states $\calS$ under the given  policy $\pi$.  Thus, $V^\pi$ exists independent of $\qq$ for arbitrary policy $\pi$, allowing us to omit the subscript $\q$ for brevity. We remark that the assumption we use is weaker than the ones in~\citet{W17m,jin2020efficiently}, as we only assume mixing bounds for deterministic stationary policies, not the randomized ones.

\subsection{Results}\label{sec:intro-result}

The main result of the paper is the following upper and lower bounds for finding an $\eps$-optimal policy for mixing AMDPs assuming a generative model access.

\begin{restatable}{theorem}{ub}\label{thm:main-ub}
There exists an algorithm that, given a mixing AMDP	with $\A$ state-action pairs and mixing time bounded by $\tmix$ and accuracy parameter $\eps\in(0,1)$, finds an $\eps$-optimal deterministic policy with probability $1-\delta$ with $O(\A\log(\A/\eps\delta)\tmix/\eps^3)$ oblivious samples.
\end{restatable}

The prior state-of-the-art sample complexity for this problem is  $\widetilde{O}(\A\tmix^2/\eps^2)$  due to the primal-dual method~\cite{jin2020efficiently}. Our method compares favorably in the following settings:
\begin{itemize}
\item $\tmix\gg\omega(1/\eps)$: In certain large-scale settings, $\tmix$ may be large and increase with dimension and whereas $\epsilon \in (0,1)$ need not. In such settings, the improved dependence on $\tmix$ achieved by our method, at the expense of larger dependence on $1/\eps$, may be desirable. 
\item one only has access to oblivious samples: Our method is the first one that uses oblivious instead of dynamic samples, which can be easier to access and cheaper to collect. Note this also implies  our method has $O(1)$ depth and is desirable for parallel computing, improving the prior parallel method for this model in~\citet{tiapkin2021parallel}.
\item one only has mixing condition on deterministic stationary policies: Our mixing condition is weaker than the standard mixing condition considered in prior work~\cite{W17m,jin2020efficiently}, which requires a $\tmix<\infty$ mixing time bound for all randomized stationary policies.
\end{itemize}

We show that our upper bound on sample complexity for mixing AMDPs is tight up to logarithmic and poly-$\eps$ factors, by proving the lower bound on oblivious samples formally as follows. 

\begin{restatable}{theorem}{lb}\label{thm:main-lb}
There are constants  $\eps_0, \delta_0\in(0,1/2)$ such that for all $\eps\in(0,\eps_0)$ and any algorithm $\mathcal{K}$, which on input mixing AMDP $(\calS,\calA,\PP,\rr)$ given by a generative model outputs a policy $\pi$ satisfying $V^{\pi}\ge V^{\pi^\star}- \eps$ with probability at least $1-\delta_0$,  $\mathcal{K}$ makes at least $\Omega(\A\tmix/\eps^2)$ deterministic oblivious queries to the generative model on some instance with $\A$ total states and mixing time at most $\tmix$. 
 \end{restatable}

\subsection{Approach}\label{sec:intro-approach}

Both our algorithm design and lower bound construction leverage ideas from research on the complexity of solving  DMDPs~\citep{azar2013minimax,SWWYY18,li2020breaking}. AMDPs are the limiting case for DMDPs with $\gamma\rightarrow 1$ and this connection has been leveraged previously to prove that Bellman equations are optimality conditions for AMDPs, obtain value iteration methods for AMDPs~\cite{puterman2014markov}, and efficiently compute stationary distributions~\cite{cohen2016faster}. However, to the best of our knowledge, quantitative application of this connection has not been performed prior for a finite $\gamma$ in studying the sample complexity of AMDPs given by a  generative model. We obtain our main theorems by bridging this gap and showing how to transfer between AMDPS and DMDPs with bounded losses in the value.

To obtain our upper bound, we prove a key fact relating AMDPs and DMDPs given the same MDP tuple $\calM\defeq (\calS,\calA,\PP,\rr)$. In Lemma~\ref{lem-mixing-discounted-closeness} we show that for any initial distribution $\qq$, the cumulative rewards as defined in~\eqref{eq:reward} satisfy $|V^\pi-(1-\gamma)V^\pi_{\qq,\gamma}|\le O((1-\gamma)\tmix)$ . This fact follows from carefully examining the matrix expressions of $V^\pi$ for AMDP and $(1-\gamma)V^\pi_{\qq,\gamma}$ for DMDP. It also utilizes the algebraic implications of the mixing property of transition matrices $\PP^\pi$. In prior work, \citet{cohen2016faster} used such algebraic techniques on mixing properties of directed graphs; and \citet{jin2020efficiently} used it on MDP analysis. With that lemma, we show that a $\gamma=1-\Theta(\eps/\tmix)$-discounted MDP well approximates an AMDP in terms of their (approximately) optimal policy, without incurring an error larger than order of $\eps$. Also, we show that it suffices to solve the corresponding DMDP to accuracy $\veps = \eps/(1-\gamma)$. By plugging in the sample complexity of most recent DMDP solvers~\citep{li2020breaking}, we thus obtain a sample complexity of 
\arxiv{
\begin{align*}
\widetilde{O}\left(\frac{\A}{(1-\gamma)^3\veps^2}\right)  \stackrel{\text{choice of }\veps}{=\joinrel=\joinrel=\joinrel=} \widetilde{O}\left(\frac{\A}{(1-\gamma)\eps^2}\right) \stackrel{\text{choice of }\gamma}{=\joinrel=\joinrel=\joinrel=} \widetilde{O}\left(\frac{\A\tmix}{\eps^3}\right).
\end{align*}}
\notarxiv{
\begin{align*}
\widetilde{O}\left(\frac{\A}{(1-\gamma)^3\veps^2}\right)  & \stackrel{\text{choice of }\veps}{=\joinrel=\joinrel=\joinrel=} \widetilde{O}\left(\frac{\A}{(1-\gamma)\eps^2}\right) \\
& \stackrel{\text{choice of }\gamma}{=\joinrel=\joinrel=\joinrel=} \widetilde{O}\left(\frac{\A\tmix}{\eps^3}\right).
\end{align*}
}

For our lower bound, we modify the hard DMDP instances considered in~\citep{azar2013minimax}. In these instances we add to most state-action pairs a small probability of $O(1/\tmix)$ to ``restart'' the Markov chain, leading to an AMDP with $\tmix$ mixing bound under the modified transition probabilities. In the class of hard instances, each state $i^1\in\xset^1$ transits to a different state $i^2_{(i^1,a^{1})}$ when the agent takes a different action $a^1\in\calA_{i^1}\defeq[K]$, and there is one action among these that contributes to a larger cumulative reward than all the rest. We show to find an $\eps$-optimal policy for the constructed AMDP, one needs to identify the correct actions for at least a constant fraction of the states $i^1\in\xset^1$, which each requires $\Omega(K\tmix/\eps^2)$ samples over all actions in $\calA_{i^1}$, giving the desired $\Omega(\A\tmix/\eps^2)$ sample lower bound.

\subsection{Previous Work}\label{sec:intro-review}

The study of sample complexities for finding approximately optimal policies for MDPs dates back to the proposal of generative models in 2000s~\citep{K03}.  Ever since, the area has seen vast progress in terms of understanding the hardness of solving different types of MDPs~\citep{azar2013minimax} and in designing efficient algorithms with improved sample complexities. Here we briefly survey advances in the complexity of computing approximately optimal policies for three typical types of MDPs given by a generative model  (see Table~\ref{tab:main} for a summary of these relevant prior results in each setup).\footnote{Another typical setting (outside the scope of this work) for studying these MDPs is to design efficient algorithms to minimize the \emph{regret} compared with the optimal policy~\cite{kearns2002near,ortner2007logarithmic,ortner2020regret}.}

\notarxiv{
\begin{table*}[t]
	\centering
	\renewcommand{\arraystretch}{1.4}
	\bgroup
	\caption{\textbf{Upper and lower bounds on sample complexity  to get $\eps$-optimal policy for different type of MDPs.} Here $\A$ denotes the total size of all state-action pairs, $\gamma$ is discount factor for DMDP, $H$-MDP corresponds to finite-horizon MDP with length $H$, $\tmix$ is mixing time for mixing AMDP, and $\tau$ is an ergodicity parameter showing up whenever the designed algorithm requires additional ergodic condition for MDP, i.e. there exists some distribution $\qq$ and $\tau>0$ satisfying $\sqrt{1/\tau}\qq\le \nnu^\pi\le\sqrt{\tau}\qq$, for any policy $\pi$ and its induced stationary distribution $\nnu^\pi$. 
	}\label{tab:main} 

	\everymath{\displaystyle}
	\begin{tabular}{c c c c}
		\toprule
		Type                        & 
		Method                    & Sample Complexity &
		Accuracy \\ 
		\midrule
				\multirow{9}{*}{\makecell{DMDP}}     & \emph{lower bound}~\citep{azar2013minimax,feng2019does}
	 &  $\widetilde{\Omega}(\A(1-\gamma)^{-3}\eps^{-2})$      & N/A \\ 
    & Empirical QVI~\citep{azar2013minimax}
	 &  $\Otilb{\A(1-\gamma)^{-3}\eps^{-2}}$      & $\eps\in(0, 1/\sqrt{(1-\gamma)|\calS|})$ \\ 
	&   Primal-Dual Method~\citep{W17d}  &   $\Otilb{\A\tau^4(1-\gamma)^{-4}\eps^{-2}}$ & $\eps\in(0, 1/(1-\gamma))$ \\ 
				&   Variance-reduced QVI~\citep{SWWYY18}      &  $\Otilb{\A(1-\gamma)^{-3}\eps^{-2}}$  & $\eps\in(0, 1)$ \\ 
		&   Empirical MDP Sampler~\citep{agarwal2019optimality}      &  $\Otilb{\A(1-\gamma)^{-3}\eps^{-2}}$   & $\eps \in(0, \sqrt{1/(1-\gamma)})$\\ 
		& Primal-Dual SMD~\cite{jin2020efficiently}                &  $\Otilb{\A(1-\gamma)^{-4}\eps^{-2}}$ & $\eps\in(0, 1/(1-\gamma))$   \\
		 & Preturbed Empirical MDP Sampler~\citep{li2020breaking}             &  $\Otilb{\A(1-\gamma)^{-3}\eps^{-2}}$ & $\eps\in(0, 1/(1-\gamma))$  \\
			\hline
		\multirow{2}{*}{\makecell{H-MDP}}   &   \emph{lower bound}~\citep{SWWYY18}      &  $\widetilde{\Omega}\left(\A H^3\eps^{-2}\right)$  & N/A \\ 
		 &   Variance-reduced QVI~\citep{SWWYY18}      &  $\Otilb{\A H^3\eps^{-2}}$  & $\eps\in(0, 1)$ \\ 
	\hline

	\multirow{2}{*}{\makecell{AMDP}} 
		& \emph{lower bound} (Theorem~\ref{thm:main-lb})  &  $\Omega\left(\A\tmix\eps^{-2}\right)$ &  N/A\\
		&  Primal-Dual Method~\citep{W17m}  &   $\Otilb{\tau^2\A\tmix^2\eps^{-2}}$ & $\eps\in(0, 1)$ \\
		& Primal-Dual SMD~\citep{jin2020efficiently}  &   $\Otilb{\A\tmix^2\eps^{-2}}$ & $\eps\in(0, 1)$
		  \\ 
		  & \textbf{Our method} (Theorem~\ref{thm:main-ub})                &    $\Otilb{\A\tmix\eps^{-3}}$ & $\eps\in(0, 1)$ 
		  \\ 

		\bottomrule
	\end{tabular}
	\egroup
\end{table*}
}
\arxiv{
\begin{table*}[t]
	\centering
	\renewcommand{\arraystretch}{1.4}
	\bgroup
	\caption{\textbf{Upper and lower bounds on sample complexity  to get $\eps$-optimal policy for different type of MDPs.} Here $\A$ denotes the total size of all state-action pairs, $\gamma$ is discount factor for DMDP, $H$-MDP corresponds to finite-horizon MDP with length $H$, and $\tmix$ is mixing time for mixing AMDP. Further, $\tau$ is an ergodicity parameter which arises in algorithms that require an additional ergodic condition for MDP which states that there exists some distribution $\qq$ and $\tau>0$ satisfying $\sqrt{1/\tau}\qq\le \nnu^\pi\le\sqrt{\tau}\qq$, for any policy $\pi$ and its induced stationary distribution $\nnu^\pi$. 
	}\label{tab:main} 

	\everymath{\displaystyle}
	\begin{tabular}{c c c c}
		\toprule
		Type                        & 
		Method                    & Sample Complexity &
		Accuracy \\ 
		\midrule
				\multirow{9}{*}{\makecell{DMDP}}     & \emph{lower bound}~\citep{azar2013minimax,feng2019does}
	 &  $\widetilde{\Omega}(\A(1-\gamma)^{-3}\eps^{-2})$      & N/A \\ 
    & Empirical QVI~\citep{azar2013minimax}
	 &  $\Otilb{\A(1-\gamma)^{-3}\eps^{-2}}$      & $\eps\in(0, \tfrac{1}{\sqrt{(1-\gamma)|\calS|}})$ \\ 
	&   Primal-Dual Method~\citep{W17d}  &   $\Otilb{\A\tau^4(1-\gamma)^{-4}\eps^{-2}}$ & $\eps\in(0, \tfrac{1}{1-\gamma})$ \\ 
				&   Variance-reduced QVI~\citep{SWWYY18}      &  $\Otilb{\A(1-\gamma)^{-3}\eps^{-2}}$  & $\eps\in(0, 1)$ \\ 
		&   Empirical MDP Sampler~\citep{agarwal2019optimality}      &  $\Otilb{\A(1-\gamma)^{-3}\eps^{-2}}$   & $\eps \in(0, \tfrac{1}{\sqrt{1-\gamma}})$\\ 
		& Primal-Dual SMD~\cite{jin2020efficiently}                &  $\Otilb{\A(1-\gamma)^{-4}\eps^{-2}}$ & $\eps\in(0, \tfrac{1}{1-\gamma})$   \\
		 & Preturbed Empirical MDP Sampler~\citep{li2020breaking}             &  $\Otilb{\A(1-\gamma)^{-3}\eps^{-2}}$ & $\eps\in(0, \tfrac{1}{1-\gamma})$  \\
			\hline
		\multirow{2}{*}{\makecell{H-MDP}}   &   \emph{lower bound}~\citep{SWWYY18}      &  $\widetilde{\Omega}\left(\A H^3\eps^{-2}\right)$  & N/A \\ 
		 &   Variance-reduced QVI~\citep{SWWYY18}      &  $\Otilb{\A H^3\eps^{-2}}$  & $\eps\in(0, 1)$ \\ 
	\hline

	\multirow{2}{*}{\makecell{AMDP}} 
		& \emph{lower bound} (Theorem~\ref{thm:main-lb})  &  $\Omega\left(\A\tmix\eps^{-2}\right)$ &  N/A\\
		&  Primal-Dual Method~\citep{W17m}  &   $\Otilb{\tau^2\A\tmix^2\eps^{-2}}$ & $\eps\in(0, 1)$ \\
		& Primal-Dual SMD~\citep{jin2020efficiently}  &   $\Otilb{\A\tmix^2\eps^{-2}}$ & $\eps\in(0, 1)$
		  \\ 
		  & \textbf{Our method} (Theorem~\ref{thm:main-ub})                &    $\Otilb{\A\tmix\eps^{-3}}$ & $\eps\in(0, 1)$ 
		  \\ 

		\bottomrule
	\end{tabular}
	\egroup
\end{table*}
}

\paragraph{DMDPs.}
~\citet{azar2013minimax} and~\citet{feng2019does} prove a lower bound of $\widetilde{\Omega}(\A/(1-\gamma)^{3}\eps^{2})$ for $\gamma$-discounted MDPs. On the upper bound side, \citet{azar2013minimax} also obtain a $Q$-value-iteration algorithm with  a (sub-)optimal sample complexity of either $\widetilde{O}(\A/(1-\gamma)^5\eps^2)$ or $\widetilde{O}(\A/(1-\gamma)^3\eps^2)$ for $\eps\in(0, 1/\sqrt{(1-\gamma)|\calS|})$. Later, a sequence of work ~\citep{SWWY18,SWWYY18} provide a variance-reduced (Q-)value iteration that has near-optimal sample complexity and runtime of~$\widetilde{O}(\A/(1-\gamma)^3\eps^2)$ for $\eps\in(0, 1)$. Their method incorporates variance reduction in estimating the value iteration step, and performs a fine-grained analysis of the error growth through the iterative process using a tight variance bound of Markov decision processes. Similarly, \citet{W19} shows that variance reduction can be applied to Q-learning and obtain a method with competing sample complexity guarantees for the same range of accuracy $\eps\in(0, 1)$. From a more statistical and less algorithmic perspective,  another work of ~\citet{agarwal2019optimality} shows that $\widetilde{O}(\A/(1-\gamma)^3\eps^2)$ samples suffice to build an empirical MDP such that the optimal policy of it yields a near-optimal policy for the original MDP. Their work also utilizes the fine-grained variance bound using Berstein inequality and extends the near-optimal sample complexity dependence to all $\eps\in(0, 1/\sqrt{1-\gamma})$. Recent work~\citep{li2020breaking} fully settled the sample complexity for DMDPs for all $\eps$, i.e. $\eps\in(0,1/(1-\gamma))$, with a perturbed empirical MDP construction. We leverage this result crucially to obtain our results in Section~\ref{sec:ub}.

\paragraph{Finite-horizon MDPs.}
For finite-horizon MDPs, cumulative reward is measured as the sum of the rewards obtained within first $H$ steps for a given finite horizon $H>0$. To the best of our knowledge, the only near-optimal algorithm given for finite-horizon MDP is in~\citet{SWWY18}. There, the authors show how to apply the near-optimal variance-reduced value iteration method to MDPs with finite horizon of length $H$. They formally prove an upper bound of $\widetilde{O}(\A H^3/\eps^{2})$ in sample complexity, for $\eps\in(0,1)$. Through reduction to DMDP lower bounds, they also obtain a lower bound of $\widetilde{\Omega}(\A H^3/\eps^2)$. 

\paragraph{AMDPs.}

Average-reward MDPs with bounded mixing time are another fundamental class of MDPs~\cite{kearns2002near,ortner2007logarithmic}, though less studied in terms of sample complexity. The first sample complexity bounds in the setting of a generative model is~\citet{W17m}, which applies a primal-dual method for the minimax problem related to the linear programming formulation, and proves an upper bound of $O(\tau^2\A\tmix^2/\eps^2)$ where $\tau$ denotes an upper bound on the ergodicity of all stationary distribution under arbitrary policies, i.e.\ there exists some distribution $\qq\in\Delta^\calS$ satisfying $\sqrt{1/\tau}\cdot\qq\le \nnu^\pi\le\sqrt{\tau}\cdot\qq$ for all policies $\pi$ and its induced stationary distribution $\nnu^\pi$. Recently, ~\citet{jin2020efficiently} design a similar primal-dual stochastic mirror descent and improve the bound to $O(\A\tmix^2/\eps^2)$, removing the ergodicity assumption through an improved analysis of the optimality conditions of the minimax problem. In contrast to value iteration, Q-learning, and sample-based methods for DMDPs, both known efficient methods for AMDPs use a linear programming formulation, dynamic sampling, a stronger mixing condition, and only compute randomized stationary policies. On the hardness side, there is no known lower bound for AMDPs with bounded mixing time.

\subsection{Notation}\label{sec:prelims}

We use unbold letters, e.g. $V$, to denote scalars, and bold letters, e.g. $\vv$ and $\PP$, to denote vectors and matrices. We use $\ee_i$ to denote the basis vector that is $1$ on coordinate $i$, and $0$ elsewhere. We use $\1_\calS$ to denote the all-ones vector in $\R^\calS$, and omit the subscript when it is clear from context. We use $\|\cdot\|_\infty$ to denote the $\ell_\infty$-norm of vectors and $\ell_\infty$-operator norm of matrices, e.g.
\arxiv{
\[
\|\vv\|_\infty \defeq \max_{s\in\calS}|\vv(s)|~\text{and}~\|\PP\|_\infty \defeq 
\max_{\norm{\xx}_\infty = 1} \norm{\PP \xx}_\infty =
\max_{s\in\calS}\sum_{s'\in\calS}|\PP(s,s')|\,.
\]
}
\notarxiv{
$\|\vv\|_\infty \defeq \max_{s\in\calS}|\vv(s)|$ and $\|\PP\|_\infty \defeq 
\max_{\norm{\xx}_\infty = 1} \norm{\PP \xx}_\infty =
\max_{s\in\calS}\sum_{s'\in\calS}|\PP(s,s')|$.
}
\section{Upper Bound}\label{sec:ub}

In this section, we prove the sample complexity upper bound for obtaining an $\eps$-optimal deterministic policy for mixing AMDPs. We first provide Lemma~\ref{lem-mixing-discounted-closeness} that relates the value of AMDPs and DMDPs under the same policy. Then we reduce solving AMDPs to DMDPs with the proper discount factor (Lemma~\ref{lem:redx-ub}) and use the state-of-the-art DMDP solver (restated in Lemma~\ref{lem:solver}) to obtain our result.

Throughout the section, we consider some mixing AMDP, and its corresponding DMDP with the same tuple $(\calS,\calA,\PP,\rr)$ and some discount factor $\gamma\in(0,1)$ to be specified.

\paragraph{Characterization of value vectors.}
  
We introduce the  value vector $\vv^\pi\in\R^\calS$ under a given policy $\pi$ for all states $s\in\calS$.  To distinguish between the value vectors of a DMDP and AMDP, we use $\vv^\pi_{\gamma}$ for DMDPs and $\vv^\pi$ for AMDPs respectively. For the discounted case, we let $\vv_{\gamma}^\pi(s) = V^\pi_{\qq = \ee_s,\gamma}$, i.e. the cumulative reward of the MDP with initial distribution only on state $s$.  We first give the following equations for computing value vectors $\vv^\pi$. These are known results widely used in literature~(see also \citet{puterman2014markov, W17m,jin2020efficiently}).

Given a tuple $(\calS,\calA,\PP,\rr)$, for DMDP with discount factor $\gamma$ and a policy $\pi$, we have 
\begin{equation}
\vv^\pi_\gamma = \sum_{t\ge 0}\gamma^{t}(\PP^\pi)^t\rr^\pi 
\enspace\text{ and }\enspace 
\linf{\vv^\pi_\gamma}\in\left[0,\frac{1}{1-\gamma}\right].\label{eq:DMDP-value}
\end{equation}
Similarly, for AMDP and a policy $\pi$ that induces stationary distribution $\nnu^\pi$, since the reward doesn't depend on initial distribution, we have
\begin{equation}
\vv^\pi = \langle \rr^\pi,\nnu^\pi\rangle\1
\enspace\text{ and }\enspace 
\linf{\vv^\pi}\in\left[0,1\right].\label{eq:DMDP-value}
\end{equation}

To put the two value vectors on the same $[0,1]$ scale, we introduce the following rescaled value vectors denoted by $\bar{\vv}^\pi$, one has
\begin{equation}\label{def:value-vector-rescaled}
\begin{aligned}
\bar{\vv}^\pi & = \vv^\pi = \lim_{T\rightarrow \infty}\frac{1}{T}\sum_{t=0}^T (\PP^\pi)^t \rr^\pi = V^\pi\cdot\1\\
 \text{ and }~~   \bar{\vv}_{\gamma}^\pi & =  (1-\gamma)\vv_{\gamma}^\pi = (1-\gamma)\left[\sum_{t=0}^\infty \gamma^t(\PP^\pi)^t\rr^\pi \right].
\end{aligned}
\end{equation}

We first state a helper lemma for $\PP^\pi$ following from the fact that it has mixing time bound $\tmix$, quoted from Lemma 23 of~\citet{cohen2016faster}.

\begin{lemma}[see Lemma 14 of~\citet{jin2020efficiently}, Lemma 23 of~\citet{cohen2016faster}]\label{lem:inf-bound}
For any policy $\pi$ with induced probabilistic transition matrix $\PP^\pi$ of mixing time $\tmix$ and stationary distribution $\nnu^\pi$ and any non-negative integer $k\ge \tmix$,
\begin{align*}
\linf{(\PP^\pi)^{k}-\1(\nu^\pi)^\top}  \le \left(\frac{1}{2}\right)^{\left\lfloor \frac{k}{\tmix}\right\rfloor}.
\end{align*}
\end{lemma}

Now we provide a lemma that bounds the $\ell_\infty$-difference of the two rescaled value vectors $\bar{\vv}^\pi$ and $\bar{\vv}_{\gamma}^\pi$ under the given policy $\pi$ and discount factor $\gamma$.

\begin{lemma}\label{lem-mixing-discounted-closeness}
Given a same MDP tuple $(\calS,\calA,\PP,\rr)$, a policy $\pi$ and some discount factor $\gamma$, rescaled value vectors $\bar{\vv}^\pi$ for AMDP and $\bar{\vv}^\pi_{\gamma}$ for DMDP as defined in~\eqref{def:value-vector-rescaled} satisfy
\[
\linf{\bar{\vv}^\pi-\bar{\vv}^\pi_{\gamma}}\le 3(1-\gamma)
\tmix.
\]
\end{lemma}
\begin{proof}
Note that
\arxiv{
\begin{align*}
\linf{\bar{\vv}^\pi-\bar{\vv}^\pi_{\gamma}}  
	&= \linfbigg{(1-\gamma)\sum_{t\ge 0}\gamma^t\langle\rr^\pi,\nnu^\pi\rangle\1-(1-\gamma)\sum_{t\ge 0}\gamma^t(\PP^\pi)^t\rr^\pi}\\
	&= (1-\gamma)\linfbigg{\sum_{t=0}^\infty \gamma^t\left[(\PP^\pi)^t-\1(\nnu^\pi)^\top\right]\rr^\pi} \\
	&\leq  (1-\gamma) \sum_{t=0}^\infty \gamma^t \linf{(\PP^\pi)^t-\1(\nnu^\pi)^\top} \cdot \linf{\rr^\pi} ~.
\end{align*}}
\notarxiv{
\begin{align*}
& \linf{\bar{\vv}^\pi-\bar{\vv}^\pi_{\gamma}} \\ 
	= & \linfbigg{(1-\gamma)\sum_{t\ge 0}\gamma^t\langle\rr^\pi,\nnu^\pi\rangle\1-(1-\gamma)\sum_{t\ge 0}\gamma^t(\PP^\pi)^t\rr^\pi}\\
	= & (1-\gamma)\linfbigg{\sum_{t=0}^\infty \gamma^t\left[(\PP^\pi)^t-\1(\nnu^\pi)^\top\right]\rr^\pi} \\
	\leq & (1-\gamma) \sum_{t=0}^\infty \gamma^t \linf{(\PP^\pi)^t-\1(\nnu^\pi)^\top} \cdot \linf{\rr^\pi} ~.
\end{align*}
}
Now, $\linf{\rr^\pi} \leq 1$ by assumption and for all $t \geq 0$ we have
\[
\linf{(\PP^\pi)^t-\1(\nnu^\pi)^\top} 
\leq \linf{(\PP^\pi)^t} + \linf{\1(\nnu^\pi)^\top}
= 2 ~. 
\]
Further for all $t \geq \tmix$ we have $\linf{(\PP^\pi)^t-\1(\nnu^\pi)^\top} \leq 2^{- \lfloor k/\tmix \rfloor}$ by  Lemma~\ref{lem:inf-bound}. Combining yields the desired bound of
\arxiv{
\begin{align*}
\linf{\bar{\vv}^\pi-\bar{\vv}^\pi_{\gamma}}  
&\leq
(1-\gamma) \sum_{t=0}^{\tmix - 1} \gamma^t \linf{(\PP^\pi)^t-\1(\nnu^\pi)^\top} 
+ (1-\gamma) \sum_{t \geq  \tmix} \gamma^t \linf{(\PP^\pi)^t-\1(\nnu^\pi)^\top}  \\
&\leq (1-\gamma) \sum_{t=0}^{ \tmix - 1} 2 \gamma^t 
+ (1-\gamma) \sum_{t \geq  \tmix} \frac{1}{2^{\lfloor k/\tmix\rfloor}}
\leq 3(1 - \gamma)\tmix\,.
\end{align*}}
\notarxiv{
\begin{align*}
& \linf{\bar{\vv}^\pi-\bar{\vv}^\pi_{\gamma}} \\ 
\leq & 
(1-\gamma) \sum_{t=0}^{\tmix - 1} \gamma^t \linf{(\PP^\pi)^t-\1(\nnu^\pi)^\top} \\
& + (1-\gamma) \sum_{t \geq  \tmix} \gamma^t \linf{(\PP^\pi)^t-\1(\nnu^\pi)^\top}  \\
\leq & (1-\gamma) \sum_{t=0}^{ \tmix - 1} 2 \gamma^t 
+ (1-\gamma) \sum_{t \geq  \tmix} \frac{1}{2^{\lfloor k/\tmix\rfloor}}\\
\leq & 3(1 - \gamma)\tmix\,.
\end{align*}}
\end{proof}

This lemma shows under the same policy, the values of AMDP and its corresponding DMDP are close up to $\eps$ when choosing the discount factor $\gamma = 1-\Theta(\eps/\tmix)$. This  allows us to formally reduce solving AMDPs to solving DMDPs with large enough discount factors in Lemma~\ref{lem:redx-ub}.

\begin{lemma}\label{lem:redx-ub}
Given an AMDP with mixing time bounded by $\tmix$, accuracy parameter $\eps\in(0,1)$, and an $\tfrac{\eps}{3(1-\gamma)}$-optimal policy $\pi$ for the corresponding DMDP with $\gamma = 1-\tfrac{\eps}{9\tmix}$, $\pi$ is also a $\eps$-optimal policy for the original AMDP.
\end{lemma}

\begin{proof}
Consider a DMDP with the same transition matrix and discount factor $\gamma  = 1-\frac{\eps}{9\tmix}$, 	we have $\linf{\bar{\vv}^\pi-\bar{\vv}^\pi_{\gamma}}\le \eps/3$ by Lemma~\ref{lem-mixing-discounted-closeness}.

Now let $\pi_d$ and $\pi_a$ denote optimal policies for the DMDP and AMDP respectively. By definition of $\pi$ one has
\arxiv{
\begin{equation}\label{def:pi-d-pi}
\linf{\vv^\pi_{\gamma}-\vv^{\pi_d}_{\gamma}}\le \frac{\eps}{3(1-\gamma)} 
\enspace\text{ or equivalently, }\enspace
\linf{\bar{\vv}^\pi_{\gamma}-\bar{\vv}^{\pi_d}_{\gamma}}\le \eps/3.
\end{equation}}
\notarxiv{
\begin{equation}\label{def:pi-d-pi}
\linf{\vv^\pi_{\gamma}-\vv^{\pi_d}_{\gamma}}\le \frac{\eps}{3(1-\gamma)} 
\enspace\Longleftrightarrow\enspace
\linf{\bar{\vv}^\pi_{\gamma}-\bar{\vv}^{\pi_d}_{\gamma}}\le \eps/3.
\end{equation}
}
Consequently, one has that entrywise,
\arxiv{
\begin{align*}
\bar{\vv}^\pi+\tfrac{\eps}{3}\1 & \stackrel{(i)}{\ge} \bar{\vv}^\pi_{\gamma}\stackrel{(ii)}{\ge} \bar{\vv}^{\pi_d}_{\gamma}-\tfrac{\eps}{3}\1\stackrel{(iii)}{\ge} \bar{\vv}^{\pi_a}_{\gamma}-\tfrac{\eps}{3}\1 \stackrel{(i)}{\ge} \left(\bar{\vv}^{\pi_a}-\tfrac{\eps}{3}\1\right)-\tfrac{\eps}{3}\1,
\end{align*}}
\notarxiv{
\begin{align*}
\bar{\vv}^\pi+\tfrac{\eps}{3}\1 & \stackrel{(i)}{\ge} \bar{\vv}^\pi_{\gamma}\stackrel{(ii)}{\ge} \bar{\vv}^{\pi_d}_{\gamma}-\tfrac{\eps}{3}\1\stackrel{(iii)}{\ge} \bar{\vv}^{\pi_a}_{\gamma}-\tfrac{\eps}{3}\1 \\
& \stackrel{(i)}{\ge} \left(\bar{\vv}^{\pi_a}-\tfrac{\eps}{3}\1\right)-\tfrac{\eps}{3}\1,
\end{align*}
}
where we use $(i)$ Lemma~\ref{lem-mixing-discounted-closeness} together with the choice of $\gamma = 1-\frac{\eps}{9\tmix}$, $(ii)$ equation~\eqref{def:pi-d-pi}, and $(iii)$ the optimality of $\pi_d$ for DMDP by definition. 

Altogether we conclude that $\bar{\vv}^\pi\ge \bar{\vv}^{\pi_a}-\eps\1$ and therefore $V^\pi\ge V^{\pi^\star}-\eps$,
i.e. $\pi$ is a $\eps$-optimal policy for the given AMDP.
\end{proof}

With the reduction, we can apply recent DMDP solvers to obtain an AMDP solver with the desired sample complexity. In order to solve the corresponding $\gamma$-discounted MDP to a desired accuracy, we use the following recent efficient DMDP solver~\cite{li2020breaking}.

\begin{lemma}[Corollary of Theorem 1 of~\citet{li2020breaking}]\label{lem:solver}
There is an algorithm that, given a $\gamma$-discounted MDP, desired accuracy $\veps\le 1/(1-\gamma)$, failure probability $0<\delta\ll 1$, outputs an $\veps$-optimal policy with probability $1-\delta$ with a number of oblivious samples bounded by
\[
\widetilde{O}\left(\frac{\A}{(1-\gamma)^3\veps^2}\log\frac{\A}{(1-\gamma)\veps\delta}\right)\,.
\]
\end{lemma}

We remark that what we state is an immediate corollary of Theorem 1 in~\citet{li2020breaking} which works for non-uniform action space per state as well by expanding the space. Now we can apply this solver to find an $\eps/(1-\gamma)$-optimal policy for $\gamma$-discounted MDP to   obtain our main result.

\ub*

\begin{proof}[Proof of Theorem~\ref{thm:main-ub}]

By Lemma~\ref{lem:redx-ub}, it suffices to solve the corresponding DMDP to $\eps/3(1-\gamma)$ accuracy with $\gamma =1-\eps/9\tmix$. Using the solver in Lemma~\ref{lem:solver} with $\veps = \eps/3(1-\gamma)$, it has sample complexity bounded by the following as stated.
\begin{align*}
& \O{\frac{\A}{(1-\gamma)^3(\eps/3(1-\gamma))^2}\log \frac{\A}{(1-\gamma)\frac{\eps}{3(1-\gamma)}\delta}} \\
= & \O{\frac{\A}{(1-\gamma)\eps^2}\log \frac{\A}{\eps\delta}} = \O{\frac{\A\tmix}{\eps^3}\log\frac{\A}{\eps\delta} },
\end{align*}
where we use the choice of $\gamma$ for the last equality. This proves the correctness of the method and yields the sample complexity bound as claimed.

\end{proof}

\notarxiv{
\begin{figure*}[ht]
\vskip 0.2in
\centering
\begin{center}
\centerline{\includegraphics[width=2\columnwidth]{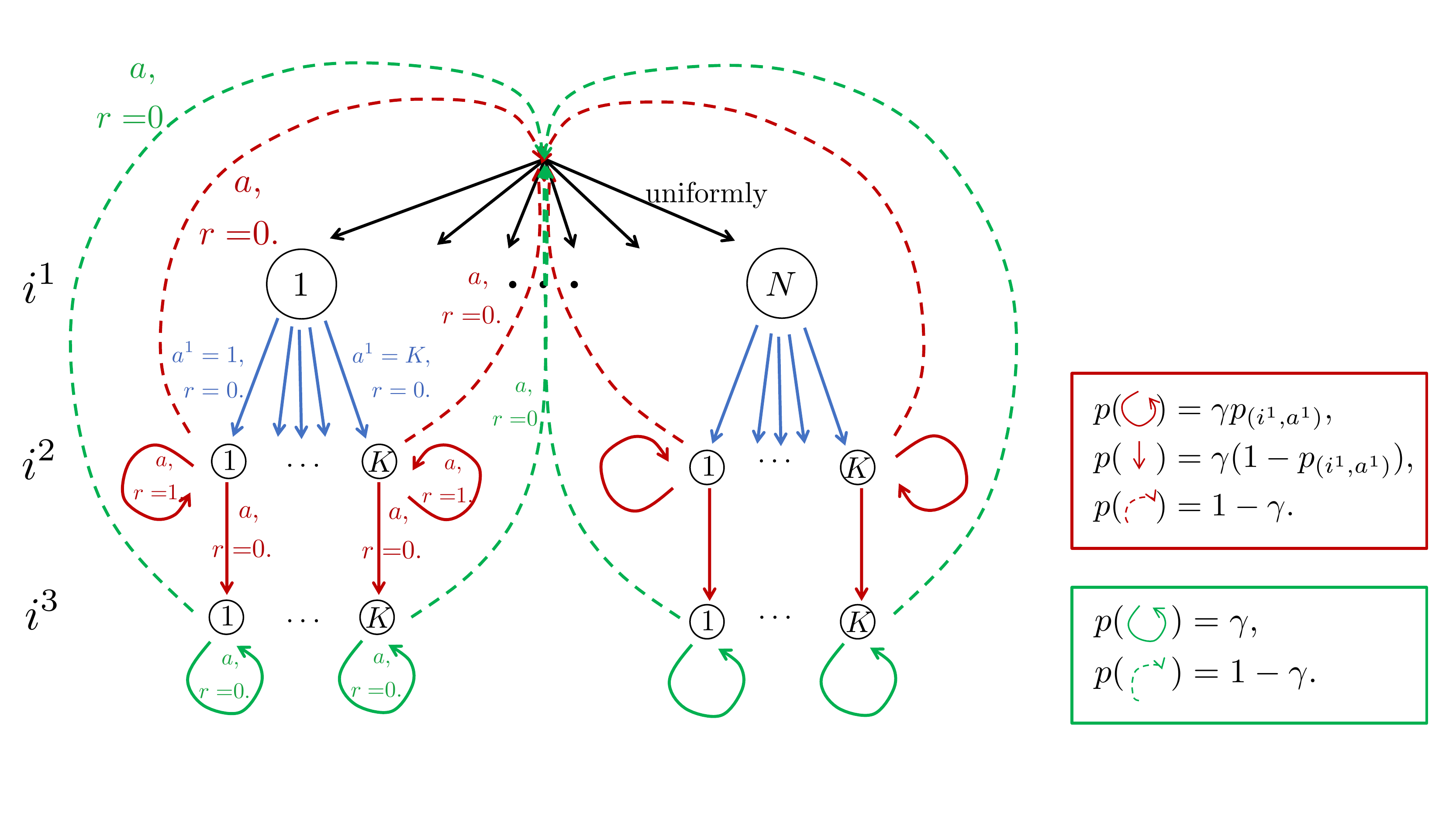}}
\caption{\textbf{AMDP lower bound hard instance illustration.} $N$ states in $\xset^1$ (corresponding to first level), $K$ action per state $i^1\in\xset^1$, $\A = O(NK)$ total state-action pairs. $\gamma\in(0,1)$, $p_{i^1,a^1}\in[0,1]$ for all $i^1\in\xset^1$, $a^1\in [K]$ are tunable parameters.}\label{fig:lb}
\end{center}
\vskip -0.2in
\end{figure*}
}

\begin{remark}[Nearly-tight $\eps$ regime.]
That \citet{li2020breaking} achieves near-optimal sample complexity for solving DMDPs for all $\veps\in(0,1/(1-\gamma))$, as opposed to the more restricted range of $\veps \in (0, 1)$ in~\citet{SWWYY18} and $\veps \in (0, 1/\sqrt{1-\gamma})$ in~\citet{agarwal2019optimality}, is key for Theorem~\ref{thm:main-ub}. This is due to the fact that we need to solve the corresponding DMDP to an accuracy of $\veps = \eps/(1-\gamma)$, which is on the order of $1/(1-\gamma)$ for constant accuracy parameter $\eps\in(0,1)$. The $\eps$ regime where we can apply prior results, i.e.  $\veps\le 1/\sqrt{1-\gamma}$, corresponds in AMDPs in the regime where our method's sample complexity is no better than that of $\widetilde{O}(\A\tmix^2/\eps^2)$ achieved by prior work~\cite{jin2020efficiently}.
\end{remark}

\section{Lower Bound}

In this section, we show a lower bound of $\Omega(\A\tmix/\eps^2)$ oblivious samples for finding an $\eps$-optimal policy for an AMDP with mixing time $\tmix$. This closes the gap (up to logarithmic and poly-$\eps$ terms) for finding an approximately-optimal policy for an AMDP given oblivious samples 
and shows that the method we propose in Theorem~\ref{thm:main-ub} with  sample complexity $\widetilde{O}(\A\tmix/\eps^{3})$ is near optimal for constant $\epsilon$. We defer some proofs in this section to Appendix~\ref{app:lb}.

To obtain this result, we provide a family of AMDP instances that we prove are difficult to solve. Our construction is similar to that given in \citet{azar2013minimax, feng2019does} for lower bounding the sample complexity of DMDPs. Formally, we consider the state space to be $\calS = \xset^1\cup\xset^2\cup\xset^3$, denoting three disjoint subsets of states on different levels (see Figure~\ref{fig:lb}). We denote the action space as $\calA_s = [K]$, for all $s = i^1\in\xset^1$, and $\calA_s = \{\text{single fixed action}\}$, for all $s\in\xset^2\cup\xset^3$.

Let  $\xset^1$ have $N$ independent states, each with $K$ independent actions. We assume for state $i^1\in\xset_1$, when taking action $a^{1}$, an agent gets to some state at second level, denoted as $i^2_{(i^1,a^1)}\in\xset^2$. At state $i^2_{(i^1,a^1)}\in\xset^2$ the agent can only take one single action after which with probability $1-\gamma$ it goes uniformly random to a state at the first level in $\xset^1$, with probability $p_{(i^1,a^1)}\gamma$ it goes back to its own state, and with probability $(1-p_{(i^1,a^1)})\gamma$ it gets to some state on the third level denoted as $i^3_{(i^1,a^1)}\in\xset^3$. At $i^3_{(i^1,a^1)}\in\xset^3$, the agent can take a single action after which with probability $1-\gamma$ it goes uniformly randomly to a state at first level in $\xset^1$ while with probability $\gamma$ it stays at the original state $i^3_{(i^1,a^1)}\in\xset^3$. A reward $1$ is generated when the agent transfers from a state in $\xset^2$ to itself, and all other transmissions generate $0$ reward. See Figure~\ref{fig:lb} for an illustration.

\arxiv{
\begin{figure*}[ht]
\vskip 0.2in
\centering
\begin{center}
\centerline{\includegraphics[width=\columnwidth]{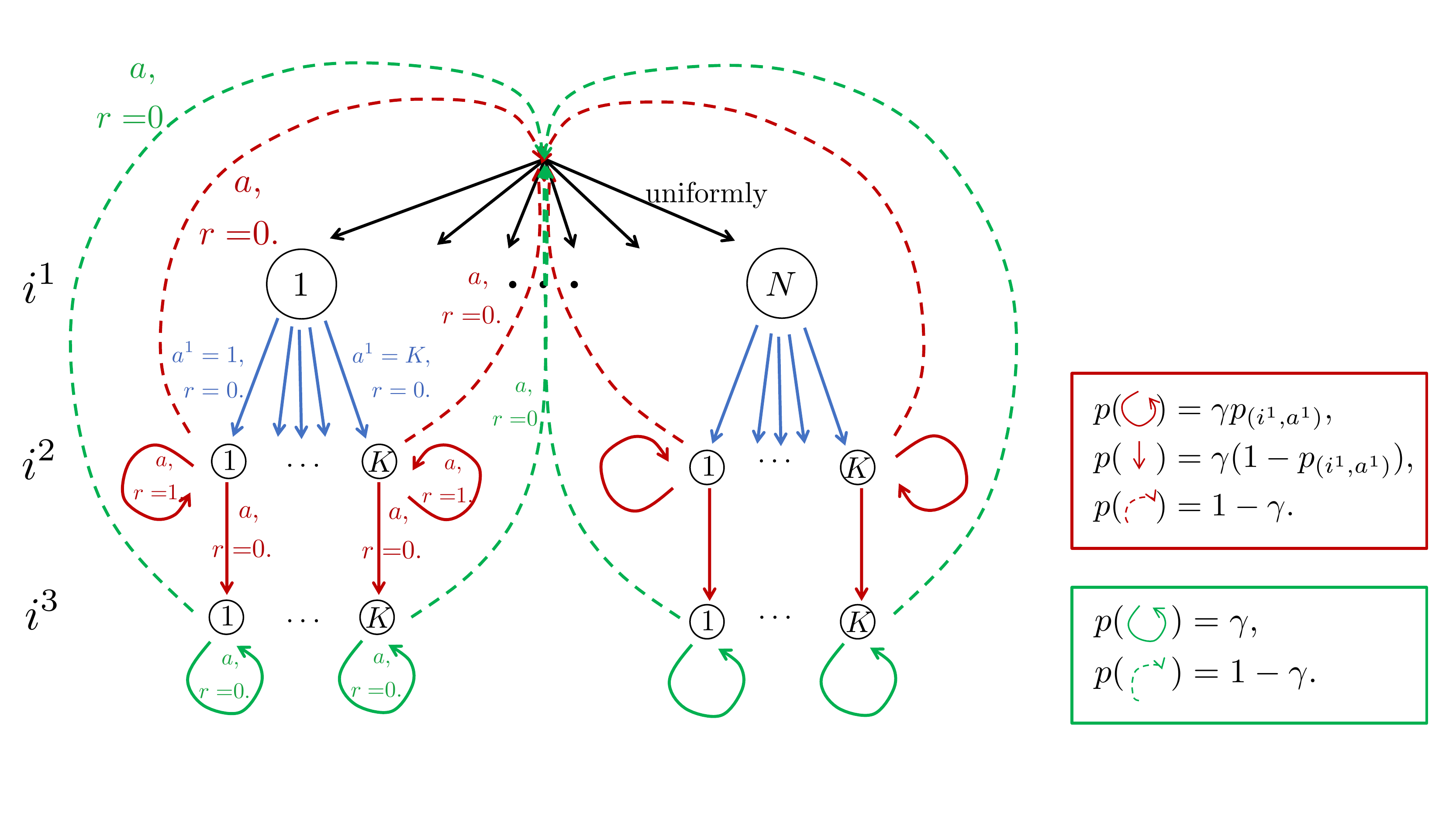}}
\caption{\textbf{AMDP lower bound hard instance illustration.} $N$ states in $\xset^1$ (corresponding to first level), $K$ action per state $i^1\in\xset^1$, $\A = O(NK)$ total state-action pairs. $\gamma\in(0,1)$, $p_{i^1,a^1}\in[0,1]$ for all $i^1\in\xset^1$, $a^1\in [K]$ are tunable parameters.}\label{fig:lb}
\end{center}
\vskip -0.2in
\end{figure*}
}

We construct the instances such that for each state-action pair $(i^1,a^1)$, a chain of length-$2$ composed of states $i^2_{(i^1,a^1)}$, $i^3_{(i^1,a^1)}$ follows. The probability $(1-\gamma)$ to go back uniformly to a state $i^1\in\xset^1$ from each chain allows the entire Markov chain to ``restart'' from $i^1$ uniformly, and ensures a $O(1/(1-\gamma))$ mixing time bound, as we show in Section~\ref{sec:lb-mix}. When in a single chain, only the transition probability $p_{(i^1,a^1)}$ of transiting from $i^2_{(i^1,a^1)}$ to itself effects the average-reward. 

To create our family of hard AMDP intances, we consider all instances such that for each $i^1\in\hat{\xset}^1$, one of the following two cases occurs: \begin{itemize}
 \item 	Case $(i)$: there is one action $k\in[K]$ that leads to transition probability $\gamma p'$, and all other actions with probability $\gamma p$; in this case the optimal action is $k$.
 \item Case $(ii)$: there is one action $a^1_\star$ that leads to transition probability $\gamma p^\star$, one action  $k\in[K]$ leading to probability $\gamma p'$, and all others leading to probability $p$; in this case the optimal action is $a^1_\star$.
 \end{itemize}
In Section~\ref{sec:lb-proof} we argue one needs to find the best action for at least a constant fraction of the states $i^1\in\xset^1$ to obtain an $O(\eps)$-optimal policy and we show this requires at least $\Omega(NK\tmix\eps^{-2})$ oblivious samples for properly chosen $p$, $p'$, and $p^\star$.

\subsection{Stationary Distribution and Mixing Time}\label{sec:lb-mix}

First, we characterize the stationary distribution under a policy $\pi$; this is useful for bounding suboptimality for policies in Section~\ref{sec:lb-proof}. 

\begin{lemma}[Stationary Distribution Characterization]~\label{lem:lb-stationary}
	Consider a policy that chooses always a fixed $a^1$ for each state in $\xset^1$. The stationary distribution $\nnu$ is in the following form: 
\begin{align*}
\nnu\left(i^1\right) & = \frac{1}{N}\cdot\frac{1-\gamma}{2-\gamma}, \forall i^1\in\xset^1;\\
\nnu\left(i^2_{(i^1,a^1)}\right) & = \frac{1}{N}\cdot\frac{1-\gamma}{(1-\gamma p_{(i^1,a^1)})(2-\gamma)},\text{ if }\pi_{i^1}(a^1) = 1,\\
\nnu\left(i^2_{(i^1,a^1)}\right) & = 0, \quad \forall a^1\in \calA,\pi_{i^1}(a^1) = 0;	\\
\nnu\left(i^3_{(i^1,a^1)}\right) & = \frac{1}{N}\cdot\frac{\gamma(1-p_{(i^1,a^1)})}{(1-\gamma p_{(i^1,a^1)})(2-\gamma)}, \text{ if }\pi_{i^1}(a^1) = 1,\\
\nnu\left(i^3_{(i^1,a^1)}\right) & = 0, \quad \forall a^1\in \calA,\pi_{i^1}(a^1) = 0.
\end{align*}

\end{lemma}

The proof of Lemma~\ref{lem:lb-stationary} follows by checking the definition of stationary distribution given the transition probabilities of the model; we defer it to Appendix~\ref{app:lb}. 

Next, we show that the mixing time of such a Markov chain under any policy $\pi$ is $O(1/(1-\gamma))$ formally; we defer the complete proof to Appendix~\ref{app:lb}. 

\begin{lemma}\label{lem:lb-mixing-time}
The AMDP constructed in this section has mixing time $\tmix\le O(1/(1-\gamma))$.
\end{lemma}

\begin{proof-sketch}
We first consider a regularized probability transition matrix in form 
\begin{equation}
\label{eq:reg_structure}
\hat{\PP} = \gamma\PP + (1-\gamma)\1\pp^\top,
\end{equation}
for some probability transition matrix $\PP$ and some density vector $\pp$.

Such a probability transition matrix induces a Markov chain where each step moves according to $\PP$ with probability $\gamma$, and restart from a random state following a fixed distribution $\pp$  with probability $1-\gamma$. After $\tilde{O}(1/(1-\gamma))$ steps the initial distribution doesn't affect where one is at because with high probability it has restarted following a fixed distribution for at least once. That the distribution isn't affected by the initial distribution ensures that the Markov chain is mixing.

Unfortunately, we cannot immediately apply this result as in our Markov chain one will only restart with certain probability when at a state in $\xset^2\cup\xset^3$ (as opposed to all states). Instead we show the $2$-step probability transition matrix admits the structure of a regularized probability transition, i.e. \eqref{eq:reg_structure}. Thus we apply the result to the $2$-step transition matrix to argue that it mixes within $O(1/(1-\gamma))$ steps, which we show implies that  the original Markov chain mixes within time steps $O(1/(1-\gamma))$, proving the statement as claimed.
\end{proof-sketch}

\subsection{Lower Bound Proofs}\label{sec:lb-proof}

Here we show a lower bound on the sample complexity for obtaining an $\eps$-optimal policy. Without loss of generality in this section we assume $N,K$ are at least some sufficiently large constants, $\eps\le 1/32$, and $\gamma\ge 1/2$.

We consider the family of AMDPs $\M$ where for each MDP instance and any fixed $i^1\in\hat{\xset}^1$ either case $(i)$ or case $(ii)$ will happen, i.e. 
\begin{equation}\label{def:p}
\begin{aligned}
\text{let}~p = \gamma,&\ p' = \gamma+\eps(1-\gamma),\ p^\star = \gamma+2\eps(1-\gamma),\\
	\text{case}~(i):& ~p_{(i^1,a^1)}= \begin{cases}
 	p'~&\text{for some}~k\in[K],\\
 	p~&\text{for any}~a^1\neq k;
 \end{cases}\\
	\text{case}~(ii):& ~p_{(i^1,a^1)}= \begin{cases}
 	p^\star~&\text{for some}~a^1_\star\in[K],\\
 	p'~&\text{for some}~k\neq a^1_\star,\\
 	p~&\text{for any}~a^1 \notin \{a^1_\star, k\}.
 \end{cases}\\
\end{aligned}
\end{equation}
Following the characterization of stationary policy in Lemma~\ref{lem:lb-stationary}, for one state $i^1\in\xset^1$ the difference of rewards when choosing a suboptimal  action that leads to the transition $\gamma p_2$ of staying at its own state in $i^2_{(i^1,a^1)}$ instead of an optimal action  that leads to the transition $\gamma p_1$,  with the policy for all other states keeping the same will be
\begin{equation}\label{eq:stationary-cals}
\begin{aligned}
& \frac{1}{N}\cdot \frac{(1-\gamma)p_1\gamma}{(1-\gamma p_1)(2-\gamma)}-\frac{1}{N}\cdot \frac{(1-\gamma)p_2\gamma}{(1-\gamma p_2)(2-\gamma)}\\
 &=  \frac{1}{N}\cdot\frac{1-\gamma}{2-\gamma}\frac{p_1(1-\gamma p_2)-p_2(1-\gamma p_1)}{(1-\gamma p_1)(1-\gamma p_2)}\\
 &= \frac{1}{N}\cdot\frac{1-\gamma}{2-\gamma}\frac{p_1-p_2}{(1-\gamma p_1)(1-\gamma p_2)} \ge \frac{2}{N}\eps,
\end{aligned}
\end{equation}
where in the last inequality we use the choice of $p_1,p_2\in\{p,p',p^\star\}$, $p_1>p_2$, and the fact that $\eps\le 1/32\le \tfrac{1}{4(1+\gamma)^2(2-\gamma)}$. Consequently, in order to obtain an $\tfrac{1}{12}\eps$-approximate deterministic policy, one must choose the optimal action, i.e. $\pi(i) = k$ for case $(i)$ and $\pi(i) = a^1_\star$  for case $(ii)$ for a subset of $i\in\calI\in[N]$ satisfying $|\calI|\ge \tfrac{23}{24}N$.

Now the key argument follows from an information-theoretical lower bound for distinguishing between two binary variables with mean $\gamma p$ and $\gamma p^\star$, formally given as follows.

\newcommand{\bin}{\mathrm{bin}}

\begin{restatable}{lemma}{leminfo}\label{lem:information}
Given a random variable $X$ drawn uniformly randomly from the family $\{\bin(\gamma p),\bin(\gamma p^\star)\}$ where $p,p^\star$ are as defined in~\eqref{def:p}. When taking fewer than $T = \Omega(1/(1-\gamma)\eps^2)$ samples of $X$, any procedure with probability $1/4$ will make a wrong prediction on which binary random variable $X$ is.
\end{restatable}

 As an immediate corollary, one can show that for any algorithm $\mathcal{K}$ taking fewer than $T$ samples on $a^1_\star$ on any randomly permuted set of actions of a state $i^1$, it must fail to distinguish between case $(i)$ and case $(ii)$ in~\eqref{def:p} with probability $1/4$.  
 
Now given any algorithm that takes fewer than $\Theta(NTK)$ deterministic~\footnote{Here deterministic means the algorithm takes fixed number of samples per state-action pair, regardless of the instance.} oblivious samples, we show the algorithm will not collect enough samples for a constant fraction of actions $a^1\in\calA_{i^1}$ of a constant fraction of states $i^1\in\xset^1$. Thus, one can consider an adversarial instance in the family $\MM$ that hides the best action uniformly randomly among the actions with insufficient samples. With constant probability the algorithm will output a suboptimal action for such $\Theta(N)$ states and output a suboptimal policy $\pi$ satisfying
 $V^\pi\le V^{\star}-\Theta(\eps)$ for some instance as we pick adversarially.
 By adjusting constants and combing with Lemma~\ref{lem:lb-mixing-time}, this implies any algorithm will need at least $\Omega(NKT) = \Omega(\A\tmix\eps^{-2})$ deterministic oblivious samples to find an $\eps$-optimal policy for all instances in the family.

 Finally, we note that our lower bound statement is applicable to any algorithms yielding deterministic policy and using deterministic oblivious sampling, which already matches our upper bound results in Section~\ref{sec:ub} nearly tightly. We present our lower bound under this setting in our main paper for clarity and simplicity. However, we believe one can extend the result to algorithms with randomized policies and using dynamic samples; we think these are interesting future directions to further strengthen the lower bound and discuss them in more detail in Appendix~\ref{app:lb-gen}.

\section{Discussion}
\label{sec:disc}

In this paper, we have shown an $\Omega(\A\tmix\eps^{-2})$ sample complexity lower bound for AMDPs with mixing time  bound $\tmix$, and a matching  upper bound (up to logarithmic and $\mathrm{poly}(1/\eps)$ factors) obtained through reduction to DMDPs . Our work suggest a few open directions which we believe would help clarify the structure of AMDPs and its connection with DMDPs:

\paragraph{Obtaining tight upper bound of sample complexity and runtime.} While the authors suspect that an $\widetilde{O}(\A\tmix\eps^{-2})$ upper bound on the required sample complexity may be attainable, it seems to require new ideas in leveraging the mixing structure of AMDP more directly, instead of reducing it to DMDPs. Further, it would be interesting to obtain algorithms with efficient running times as has been shown for DMDPs~\citep{SWWYY18} in certain $\epsilon$-regimes.
 \paragraph{Relaxing the mixing bound assumption.} In certain cases, assuming global mixing time bound for all policies, even for all deterministic stationary policies (as we do in the paper), can be restrictive. We ask if it is possible to obtain sample complexity dependence in terms of the mixing time of the optimal policy, or in terms of some alternative parameters like diameter~\citep{jaksch2010near}, or bias span~\citep{bartlett2012regal,fruit2018efficient} that can be smaller than $\tmix$ for certain types of AMDPs.

\section*{Acknowledgment}
The authors thank Ron Amit and Lin F. Yang for helpful conversations. YJ was supported by Stanford Graduate Fellowship. AS was supported in part by a Microsoft Research Faculty Fellowship, NSF CAREER Award
 CCF-1844855, NSF Grant CCF-1955039, a PayPal research award, and a Sloan Research Fellowship.

\newpage

\arxiv{
\bibliographystyle{abbrvnat}

}
\notarxiv{

\bibliographystyle{icml2021}
}

\onecolumn
\newpage
\appendix

\part*{Supplementary material}

\section{Omitted Proofs for the Lower Bound}~\label{app:lb}

In this section we give a detailed proof for the lower bound argument, formally we will prove the theorem:

\lb*

First we give the proof for the concrete characterization of the stationary distribution given a deterministic policy for our AMDP instance.

\begin{proof}[Proof of Lemma~\ref{lem:lb-stationary}]
For simplicity we write $\pi(i^1)=a^1$ to denote the action $a^1$ we take at each state $i^1$. To verify this is the stationary distribution for all states, one can check the equality from the definition of stationary distribution. Given $\PP$ as the transitional matrix of the given Markov chain, and $\nnu$ as the distribution specified above, one can see for all states $i^1\in\xset^1$, it holds that $\forall i^1\in\xset^1$,
\begin{align*}
 \nnu(i^1) &= \frac{1}{N}\cdot\frac{1-\gamma}{2-\gamma} = N\cdot\frac{1-\gamma}{N}\left(\frac{1}{2-\gamma}\right)\\
 &=  \sum_{i^1\in\xset^1}\left(\nnu\left(i^2_{(i^1,\pi(i^1))}\right)\cdot\frac{1-\gamma}{N}+\nnu\left(i^3_{(i^1,\pi(i^1))}\right)\cdot\frac{1-\gamma}{N}\right).
\end{align*}

Also, for particular state in $\xset^2$, if $a^1\neq\pi(i^1)$, the distribution must be $0$; if $a^1=\pi(i^1)$, we have
\begin{align*}
\nnu\left(i^2_{(i^1,a^1)}\right)  = \frac{1}{N}\cdot \frac{1-\gamma}{(1-\gamma p_{(i^1,a^1)})(2-\gamma)} 
 = \nnu\left(i^1\right)\cdot 1+\nnu\left(i^2_{(i^1,a^1)}\right)\cdot \gamma p_{(i^1,a^1)}.
\end{align*}

Similarly for states in $\xset^3$, apart from $0$ if $a^1=\pi(i^1)$ we have
\begin{align*}
 \nnu\left(i^3_{(i^1,a^1)}\right) = \frac{1}{N}\cdot \frac{\gamma(1-p_{(i^1,a^1)})}{(1-\gamma p_{(i^1,a^1)})(2-\gamma)}
=  \nnu\left(i^2_{(i^1,a^1)}\right)\cdot \left(1-p_{(i^1,a^1)}\right)\gamma+\nnu\left(i^3_{(i^1,a^1)}\right)\cdot \gamma.
\end{align*}
\end{proof}

Now we prove Lemma~\ref{lem:lb-mixing-time} on bounds of mixing time for our constructed instance formally. To do that, we first prove the following lemma that offers intuition in bounding the mixing time of our interested transition matrix. We use $\|\vv\|_1 = \sum_{s\in\calS}|\vv(s)|$ as the standard $\ell_1$ norm of vectors.

\begin{lemma}\label{lem:helper-1}
Given a Markov chain with transition probability matrix $\PP$, a probability density vector $\pp\in\Delta^\calS$, and $\gamma\in(1/2,1)$, the Markov chain induced by $\hat{\PP} = \gamma\PP+(1-\gamma)\1\pp^\top$ has mixing time bound $
\tmix\le O(1/(1-\gamma))$.
\end{lemma}
\begin{proof}
By induction on $t$ we have that for all for all $t \geq 1$,
\[\hat{\PP}^t = \gamma^t\PP^t+\1\hat{\pp}_t^\top,~~\text{ for some }~~\hat{\pp}_t.\]
Now, for any arbitrary initial distribution $\qq_1$, $\qq_2$ we have that
\begin{align*}
\lonebigg{\qq_1^\top \hat{\PP}^t-\qq_2^\top \hat{\PP}^t} & =\lonebigg{\left(\qq_1-\qq_2\right)^\top\left(\gamma^t\PP^t+\1\hat{\pp}_t^\top\right)}\\
& = \lonebigg{\left(\qq_1-\qq_2\right)^\top\gamma^t\PP^t}\le 2\gamma^t \le  2\exp(-t(1-\gamma)).
\end{align*}

By Perron-Frobenius Theorem, we know there must exists some stationary distribution $\nnu$ satisfying $\nnu^\top \hat{\PP} = \nnu^\top$. Now taking $\qq_2 = \nnu$ and some arbitrary $\qq_1\in\Delta^\calS$, we have for $t\ge t_0\defeq \lceil\frac{\log 4}{1-\gamma}\rceil$, 
\begin{align*}
\lonebigg{\qq_1^\top \hat{\PP}^t-\nnu^\top} \le 2\exp(-t(1-\gamma))\le 1/2.
\end{align*}

Thus one can immediately conclude by definition of the mixing MDP that the induced Markov chain has a unique stationary distribution and that its mixing time is bounded by $O(1/(1-\gamma))$.
\end{proof}

Now we provide the formal proof for Lemma~\ref{lem:lb-mixing-time}.

\begin{proof}[Proof of Lemma~\ref{lem:lb-mixing-time}]

Now for a fixed policy we consider the probability transition matrix $\PP$ corresponding to the Markov chain of our problem~(see Figure~\ref{fig:lb}), we write the block-wise decomposition form of the matrix as 
\[
\PP = \begin{pmatrix}
 \0 & \PP_1 \\
 (1-\gamma)\cdot\frac{1}{N}\1\cdot\1^\top	& \gamma \PP_2
 \end{pmatrix},
\]
where the first block $\0\in\R^{N\times N}$ corresponds  to the probability transition matrix from states in $\xset^1$ to $\xset^1$, the second block $\PP_1$ corresponds to transition matrix from states in $\xset^1$  to states in $\xset^2\cup\xset^3$, and similarly for the rest.

By considering the $2$-step transition probability matrix we have 
\begin{align*}
\PP^2 = & \begin{pmatrix}
(1-\gamma)\frac{1}{N}\1\cdot\1^\top & \gamma\PP_1\PP_2 \\
 \gamma(1-\gamma)\cdot\frac{1}{N}\1\cdot\1^\top	&  (1-\gamma)\cdot\frac{1}{N}\1\cdot\1^\top + \gamma^2 \PP_2^2
 \end{pmatrix}\\
  = &  
 (1-\gamma+\gamma^2) \NN_0+\gamma(1-\gamma)\NN,\\
 \text{ where we define }
 \NN_0 = & \begin{pmatrix}
 \frac{(1-\gamma)^2}{1-\gamma+\gamma^2}\frac{1}{N}\1\cdot\1^\top & \frac{\gamma}{1-\gamma+\gamma^2}\PP_1\PP_2\\
 \0 &  \frac{1-\gamma}{1-\gamma+\gamma^2}\cdot\frac{1}{N}\1\cdot\1^\top + \frac{\gamma^2}{1-\gamma+\gamma^2} \PP_2^2	
 \end{pmatrix}\\
 ~\text{and}~
\NN = & \begin{pmatrix}
\frac{1}{N}\1\1^\top & \0	
\end{pmatrix} = \1\pp^\top.
\end{align*}

Note $\NN_0$ is a probability transition matrix and therefore $\linf{\NN_0}\le 1$. As a result, we can apply Lemma~\ref{lem:helper-1} with $\gamma' = 1-\gamma(1-\gamma)$, $\gamma\ge 1/2$ to conclude that for arbitrary two distributions $\qq_1$, $\qq_2$, we have for all $t\ge 1$, 
\begin{align*}
\lonebigg{\qq_1^\top \PP^{2t}-\qq_2^\top \PP^{2t}} & =\lonebigg{\left(\qq_1-\qq_2\right)^\top\gamma^t\NN_0^t}\le 2(1-\gamma(1-\gamma))^t \le  2\exp(-t\gamma(1-\gamma));\\
\lonebigg{\qq_1^\top \PP^{2t+1}-\qq_2^\top \PP^{2t+1}} & \le \lonebigg{\qq_1^\top \PP^{2t}-\qq_2^\top \PP^{2t}}\norm{\PP}_\infty\le \lonebigg{\qq_1^\top \PP^{2t}-\qq_2^\top \PP^{2t}} \le  2\exp(-t\gamma(1-\gamma))\,.
\end{align*}
Thus one can conclude that the mixing time for this Markov chain induced by the form of probability transition matrix we have is bounded by $\tmix\le O(1/\gamma(1-\gamma))= O(1/(1-\gamma))$.
\end{proof}

To prove Theorem~\ref{thm:main-lb}, we first give a simple information theoretic lower bound for distinguishing two binary random variables $\pp_1 = \mathrm{bin}(\gamma p^\star)$ and $\pp_2=\mathrm{bin}(\gamma p)$~\cite{yu1997assouad}.

\leminfo*

\begin{proof}[Proof of Lemma~\ref{lem:information}]
Here we use of the KL divergence of random variables, defined for the two discretized probability variables $\pp_1$ and $\pp_2$ as

\[
\mathrm{KL}(\pp_1||\pp_2) = \sum_{i}\pp_1(i)\log (\pp_1(i)/\pp_2(i)).
\]
Now for the given two binary random variables $\pp_1=\mathrm{bin}(\gamma p^\star),\pp_2=\mathrm{bin}(\gamma p)$,
\[
\mathrm{KL}(\pp_1||\pp_2) =\gamma(\gamma+\eps(1-\gamma))\log\left(\frac{\gamma+\eps(1-\gamma)}{\gamma}\right)+ \gamma(1-\gamma-\eps(1-\gamma))\log\left(\frac{1-\gamma-\eps(1-\gamma)}{1-\gamma}\right).
\]
We know by Le Cam's inequality~\citep{yu1997assouad} that for every testing procedure $\Psi$, if the environment chooses a binary distribution $V$ randomly, let $\Psi(X)$ be the output if the identified distribution of procedure $\Psi$ under a sequence of observations $X$, let $\P$ be taking over the randomness coming from both $X,V$,  we have 
\[
\P(\Psi(X)\neq V)\ge \frac{1}{2}\left(1-\sqrt{\frac{t}{2}\mathrm{KL}(\pp_1||\pp_2)}\right)= \frac{1}{2}\left(1-O(\sqrt{\eps^2(1-\gamma)t})\right)
\]
Thus when one take samples  $T \le  \Omega(1/(1-\gamma)\eps^2)$, for any testing procedure, with probability at least $1/4$ it would make a wrong prediction. 
\end{proof}

Now for $K\ge 3$, we consider a family of instances of binary variables $\SSS_K$ as follows: Given $p,p',p^\star$ defined as in~\eqref{def:p},
\begin{equation} \label{eq:def-disn}
\begin{aligned}
	\text{let}~&~\nu = (\nu_1,\cdots,\nu_K)~\text{where}~\nu_K = \mathrm{bin}(\gamma p'), \nu_i = \mathrm{bin}(\gamma p)~\text{for}~i\le K-1,\text{ and}\\
	~\text{let}~&~\nu'= (\nu'_1,\cdots,\nu'_K)~\text{ where }~\nu'_K = \mathrm{bin}(\gamma p'), \nu'_{K-1} = \mathrm{bin}(\gamma p^\star), \nu'_i = \mathrm{bin}(\gamma p)~\text{ for }~i\le K-2.
\end{aligned}
\end{equation}

We consider $\nu,\nu'$ under all random permutations $\sigma$ over $[K]$ such that $\SSS_K = \{\nu_{\sigma(i)},\nu'_{\sigma(i)}:i\in[K],~\sigma~\text{is random permutation over }[K]\}$. One has the following corollary from the information-theoretic lower bound of identifying binary variables in Lemma~\ref{lem:information}.

\begin{corollary}\label{cor:information}
	Consider a family of instances $\SSS_K$ as defined above. For any testing procedure taking oblivious samples of each binary variable from an arbitrary permutation of one of the two instances $\nu_\sigma$, $\nu'_\sigma$, and outputs a prediction of which index corresponds to a random variable with the highest mean, if it takes samples fewer than $T$ with $T = \Omega(1/(1-\gamma)\eps^2)$ in Lemma~\ref{lem:information}  on $\nu_{K-1}$, it must make wrong predictions for at least one of the instance with probability $1/4$. 
\end{corollary}
\begin{proof}
We prove by contradiction. Suppose there is a testing procedure $\mathcal{K}$ and two instances $\nu_{\sigma}$, $\nu_{\sigma'}$ in the family $\SSS_K$ that $\mathcal{K}$ can always make the correct prediction with probability $0.75$ for both. Then, we define the following procedure $\mathcal{K}'$ for testing binary variable $X$ given fewer than $T$ samples from $X$ and the permutation $\sigma$. Note this is more information and will not make the following problem harder for procedure $\mathcal{K}$ to solve. Let $\nu'_{K-1} = X$ with its own ($\le T$) samples, and all others be $\mathrm{bin}(\gamma p)$ or $\mathrm{\bin}(\gamma p')$ with auxiliary samples generated from their own distribution as in $\nu'_{\sigma}$, by assumption we know if $X = \bin(\gamma p^\star)$, then applying procedure $\mathcal{K}$ it outputs $\sigma_1^{-1}(K-1)$ as the index of the binary variable with highest mean with probability $0.75$. Similarly, if $X = \bin(\gamma p)$, then procedure $\mathcal{K}$ outputs $\sigma_1^{-1}(K)$ as the corresponding binary variable index  with highest mean with probability $0.75$. Thus, we obtain a procedure $\mathcal{K}'$ using $\mathcal{K}$ that can identify the distribution of $X$ correctly within $T$ samples with probability $0.75$, contradicting Lemma~\ref{lem:information}.
\end{proof}

\begin{proof}[Proof of Theorem~\ref{thm:main-lb}]
Consider an arbitrary algorithm $\mathcal{K}$ which takes a total number of deterministic oblivious samples $\frac{N}{6}KT$ where $T$ is as defined in~Lemma~\ref{lem:information} and let $N$ be divisible by $24$, $K$ be divisible by $2$. There is a subset $\hat{\xset}^1\subseteq\xset^1$ such that $|\hat{\xset}^1|= \frac{2N}{3}$  and that for all $i^1\in\hat{\xset}^1$, the algorithm gets at most $T$ samples from each of the action in some subset $\hat{\calA}_{i^1}$ of actions with size $|\hat{\calA}_{i^1}| = K/2$. Now consider a family of MDP instances $\calM\in\MM$ where for each $i^1\in\hat{\xset}^1$, the actions in the subset $\hat{\calA}_{i^1}$ and the transition probabilities at $i^2_{(i^1,a^1)}$ of staying at itself under these actions are characterized fully by one of permutations $\nu$ or $\nu'$ defined as in~\eqref{eq:def-disn}. When induced by $\nu$, the optimal action leads to transition probability $\gamma p'$ and sub-optimal actions leads to $\gamma p$; when induced by $\nu'$, the optimal action leads to transition probability $\gamma p^\star$ and sub-optimal ones leads to either $p$ or $p'$. All other actions in $\calA_{i^1}\setminus \hat{\calA}_{i^1}$ for $i^1\in\hat{\xset}^1$ and all actions for $i^1\in\xset\setminus\hat{\xset}^1$ have transition probability $\gamma p$ of staying at their own states in level $i^2$. 

Given any algorithm $\mathcal{K}$ that takes deterministic oblivious samples on each state-action pair,  denote $Y_i = \1_{\{\mathcal{K}\text{ didn't output the optimal action for state i}\}}$, $\forall i\in\hat{\xset}^1$ as a random variable. Consider a random instance where for each $i^1\in\hat{\xset}^1$, the actions follow some permutation $\sigma$ and one of $\{\nu_{\sigma},\nu_{\sigma'}\}$ uniformly randomly. By Corollary~\ref{cor:information} we know that $\E Y_i \ge 1/8$, and thus we have 
\[
\E\left[\sum_{i\in\hat{\xset}^1}Y_i\right]\ge \frac{1}{12}N~~\stackrel{\text{Markov's inequality}}{\implies}\mathbb{P}\left(\sum_{i\in\hat{\xset}^1}Y_i\le \frac{1}{24}N\right)\le \frac{11N}{12}\cdot\frac{24}{23N} =\frac{22}{23}
\]
Thus, there exists an instance for which when algorithm takes fewer than $\frac{N}{6}KT$ oblivious samples, with probability at least $1/23$ the algorithm outputs a policy with suboptimal actions on more than $\tfrac{1}{24}N$ states of $i^1\in\hat{\xset}^1$

However, note that when the algorithm outputs a suboptimal action for a state, it will incur a loss of at least 
\begin{align*}
& \frac{1}{N}\cdot \frac{(1-\gamma)p'\gamma}{(1-\gamma p')(2-\gamma)}-\frac{1}{N}\cdot \frac{(1-\gamma)p\gamma}{(1-\gamma p)(2-\gamma)}\\
 &=  \frac{1}{N}\cdot\frac{1-\gamma}{2-\gamma}\frac{p'(1-\gamma p)-p(1-\gamma p')}{(1-\gamma p')(1-\gamma p)}\\
 &= \frac{1}{N}\cdot\frac{1-\gamma}{2-\gamma}\frac{p'-p}{(1-\gamma p')(1-\gamma p)} \ge \frac{2}{N}\eps,
\end{align*}
in the average reward. Thus, we conclude with probability at least $1/23$ the algorithm will output a $\tfrac{1}{12}\eps$-suboptimal policy $\pi$ satisfying $V^\pi\le V^{\pi^\star}-\tfrac{1}{12}\eps$ on some instance $\calM_0\in\MM$.

Thus by adjusting constant of $\eps$, for some constant $\delta_0=1/23$ and  large enough $N$ we conclude that the number of necessary samples is $\Omega(NKT) = \Omega(NK/(1-\gamma)\eps^2) = \Omega(NK\tmix/\eps^2)$ by definition of $T$ in Lemma~\ref{lem:information} and~mixing time bound in~Lemma~\ref{lem:lb-mixing-time}, as stated in the theorem.
\end{proof}

\section{Generalization of the Lower Bound}\label{app:lb-gen}

In this section, we discuss some potential generalizations of our lower bound result. We first show one can fully characterize all randomized policies: Consider a policy that at state $i^1$ chooses $a^1$ with probability $\pi_{i^1}(a^1)$. By definition, $\sum_{a^1}\pi_{i^1}(a^1)=1,\forall i^1\in\xset^1$, following symmetry of actions and the structure of sequential independent chains in our construction, the stationary distribution $\nnu(i^2_{(i^1,a^1)})\propto \pi_{i^1}(a^1)$ for all $a^1\in[K]$. Thus similar to  Lemma~\ref{lem:lb-stationary} of stationary distribution for deterministic policies, we have that the stationary distribution of a given randomized policy $\pi$ is (let $i^1\in\xset^1$, $a^1\in\xset^1$)
\begin{align*}
\nnu\left(i^1\right) & = \frac{1}{N}\cdot\frac{1-\gamma}{2-\gamma}, \forall i^1\in\xset^1;\\
\nnu\left(i^2_{(i^1,a^1)}\right) & = \frac{\pi_{i^1}(a^1)}{N}\cdot\frac{1-\gamma}{(1-\gamma p_{(i^1,a^1)})(2-\gamma)},  \forall i^1, a^1;	\\
\nnu\left(i^3_{(i^1,a^1)}\right) & = \frac{\pi_{i^1}(a^1)}{N}\cdot\frac{\gamma(1-p_{(i^1,a^1)})}{(1-\gamma p_{(i^1,a^1)})(2-\gamma)}, \forall i^1, a^1.
\end{align*}

Combining this structure of stationary distribution together with Corollary~\ref{cor:information} gives the following generalization of our lower bound to all randomized policies.

\begin{restatable}[Generalization to randomized policies]{theorem}{lb-gen}\label{thm:main-lb-gen}
There are constants  $\eps_0, \delta_0\in(0,1/2)$ such that for all $\eps\in(0,\eps_0)$ and any \emph{randomized} algorithm $\mathcal{K}$, which on input mixing AMDP $(\calS,\calA,\PP,\rr)$ given by a generative model outputs a \emph{randomized} policy $\pi$ satisfying $ V^{\pi}\ge V^{\pi^\star}- \eps$ with probability at least $1-\delta_0$,  $\mathcal{K}$ makes at least $\Omega(\A\tmix/\eps^2)$ deterministic oblivious queries to the generative model on some instance with $\A$ total states and mixing time at most $\tmix$. 
 \end{restatable}
 
 \begin{proof}
We define $T,\hat{\xset}^1, \hat{\calA}_{i^1}$ as in Theorem~\ref{thm:main-lb}. Consider any procedure that outputs a randomized policy for a single state $i^1\in\xset^1$. By Corollary~\ref{cor:information}, we know with probability $\frac{1}{16}$ it must output a randomized policy satisfying $\pi_{i^1}(\text{optimal action})\le 1/15$, as otherwise one can round the randomized policy to a deterministic one with larger than $3/4$ success probability. Using the structure of stationary distribution under randomized policy, we note that whenever the algorithm outputs a randomized policy with $\pi_{i^1}(\text{optimal action})\le 14/15$ for some state $i^1$, it incurs an average loss in the reward as
\begin{equation}
\left(1-\frac{14}{15}\right)\cdot\frac{2\eps}{N} = \frac{2\eps}{15N}.\label{eq:loss-per-state-gen}
\end{equation}

Using the similar argument as in Theorem~\ref{thm:main-lb} and by adjusting constants, this proves the generalized lower bound.
 \end{proof}

It would also be interesting to consider more sophisticated sampling schemes to generalize our lower bound result. For instance, we believe we can handle algorithms with randomized oblivious samples, by considering a fixed permutation and choosing $\nu_{K-1} = \gamma p$ or $\gamma p^\star$ uniformly at random~\eqref{eq:def-disn} in constructing the hard instance.  Even more broadly, we conjecture that a lower bound result for any algorithms with adaptive samples is achievable through a more careful argument. In particular, the information-theoretic lower bounds shown for DMDPs in~\citet{azar2013minimax,feng2019does} use dynamic sampling, i.e. when the samples are generated iteratively and might depend on the history observation, might be adaptable to AMDPs as well. Similar to our current proof strategy, that would crucially rely on the structure of our constructed MDP and the independence between states $i^1\in\xset^1$. 

Finally, we note that given our upper bounds in Section~\ref{sec:ub}, argument for any algorithms yielding deterministic policy and using deterministic oblivious sampling already matches our upper bound results nearly tightly. So we present our lower bound under this setting in our main paper for clarity and simplicity. However, we still think generalizations along these lines are helpful to fully characterize the hardness of solving AMDPs.

\end{document}